\documentclass[10pt,twocolumn,letterpaper]{article}

\usepackage{wacv}
\usepackage{times}
\usepackage{epsfig}
\usepackage{graphicx}
\usepackage{subcaption}
\usepackage{amsmath}
\usepackage{amsthm}
\usepackage{amssymb}
\usepackage{amsfonts}
\usepackage{bm}
\usepackage{float}
\usepackage[inline]{enumitem}
\usepackage{booktabs}
\usepackage{diagbox}
\usepackage{multirow}
\usepackage{multicol}
\usepackage{lipsum}
\usepackage{mwe}
\usepackage{pifont}

\def\Figref#1{Figure~\ref{#1}}

\def\Secref#1{Section~\ref{#1}}

\def\eqref#1{equation~\ref{#1}}
\def\Eqref#1{Equation~\ref{#1}}

\def\1{\bm{1}}

\DeclareMathAlphabet{\mathsfit}{\encodingdefault}{\sfdefault}{m}{sl}
\SetMathAlphabet{\mathsfit}{bold}{\encodingdefault}{\sfdefault}{bx}{n}

\newtheorem{theorem}{Theorem}
\newtheorem{lemma}{Lemma}

\wacvfinalcopy 

\ifwacvfinal
\usepackage[breaklinks=true,bookmarks=false]{hyperref}
\else
\usepackage[pagebackref=true,breaklinks=true,colorlinks,bookmarks=false]{hyperref}
\fi

\newcommand{\cmark}{\ding{51}}%
\newcommand{\xmark}{\ding{55}}%

\makeatletter
\newcommand{\printfnsymbol}[1]{\textsuperscript{\@fnsymbol{#1}}}
\makeatother

\makeatletter
\def\blfootnote{\gdef\@thefnmark{}\@footnotetext}
\makeatother

\begin{document}

\title{Scale Equivariance Improves Siamese Tracking}

\author{Ivan Sosnovik\thanks{equal contribution} 
\quad Artem Moskalev \printfnsymbol{1}
\quad Arnold Smeulders \\
UvA-Bosch Delta Lab\\
University of Amsterdam, Netherlands\\
{\tt\small \{i.sosnovik, a.moskalev, a.w.m.smeulders\}@uva.nl}
}

\maketitle

\begin{abstract}
Siamese trackers turn tracking into similarity estimation between a template and the candidate regions in the frame. Mathematically, one of the key ingredients of success of the similarity function is translation equivariance. Non-translation-equivariant architectures induce a positional bias during training, so the location of the target will be hard to recover from the feature space. In real life scenarios, objects undergoe various transformations other than translation, such as rotation or scaling. Unless the model has an internal mechanism to handle them, the similarity may degrade. In this paper, we focus on scaling and we aim to equip the Siamese network with additional built-in scale equivariance to capture the natural variations of the target a priori. We develop the theory for scale-equivariant Siamese trackers, and provide a simple recipe for how to make a wide range of existing trackers scale-equivariant. We present SE-SiamFC, a scale-equivariant variant of SiamFC built according to the recipe. We conduct experiments on OTB and VOT benchmarks and on the synthetically generated T-MNIST and S-MNIST  datasets. We demonstrate that a built-in additional scale equivariance is useful for visual object tracking.

\blfootnote{Source code: \url{https://github.com/isosnovik/SiamSE}}
\end{abstract}

\section{Introduction}
\label{sec:intro}
Siamese trackers turn tracking into similarity estimation between a template and the candidate regions in the frame. The Siamese networks are successful because the similarity function is powerful: it can learn the variances of appearance very effectively, to such a degree that even the association of the frontside of an unknown object to its backside is usually successful. And, once the similarity is effective, the location of the candidate region is reduced to simply selecting the most similar candidate.

Mathematically, one of the key ingredients of the success of the similarity function is translation \textit{equivariance}, i.e. a translation in the input image is to result in the proportional translation in feature space. Non-translation-equivariant architectures will induce a positional bias during training, so the location of the target will be hard to recover from the feature space \cite{li2018siamrpn++, zhang2019deeper}. In real-life scenarios, the target will undergo more transformations than just translation, and, unless the network has an internal mechanism to handle them, the similarity may degrade. We start from the position that equivariance to common transformations should be the guiding principle in designing conceptually simple yet robust trackers. To that end, we focus on scale equivariance for trackers in this paper.

\ifwacvfinal
\begin{figure}[t]
  \centering
    \includegraphics[width=0.39\textwidth]{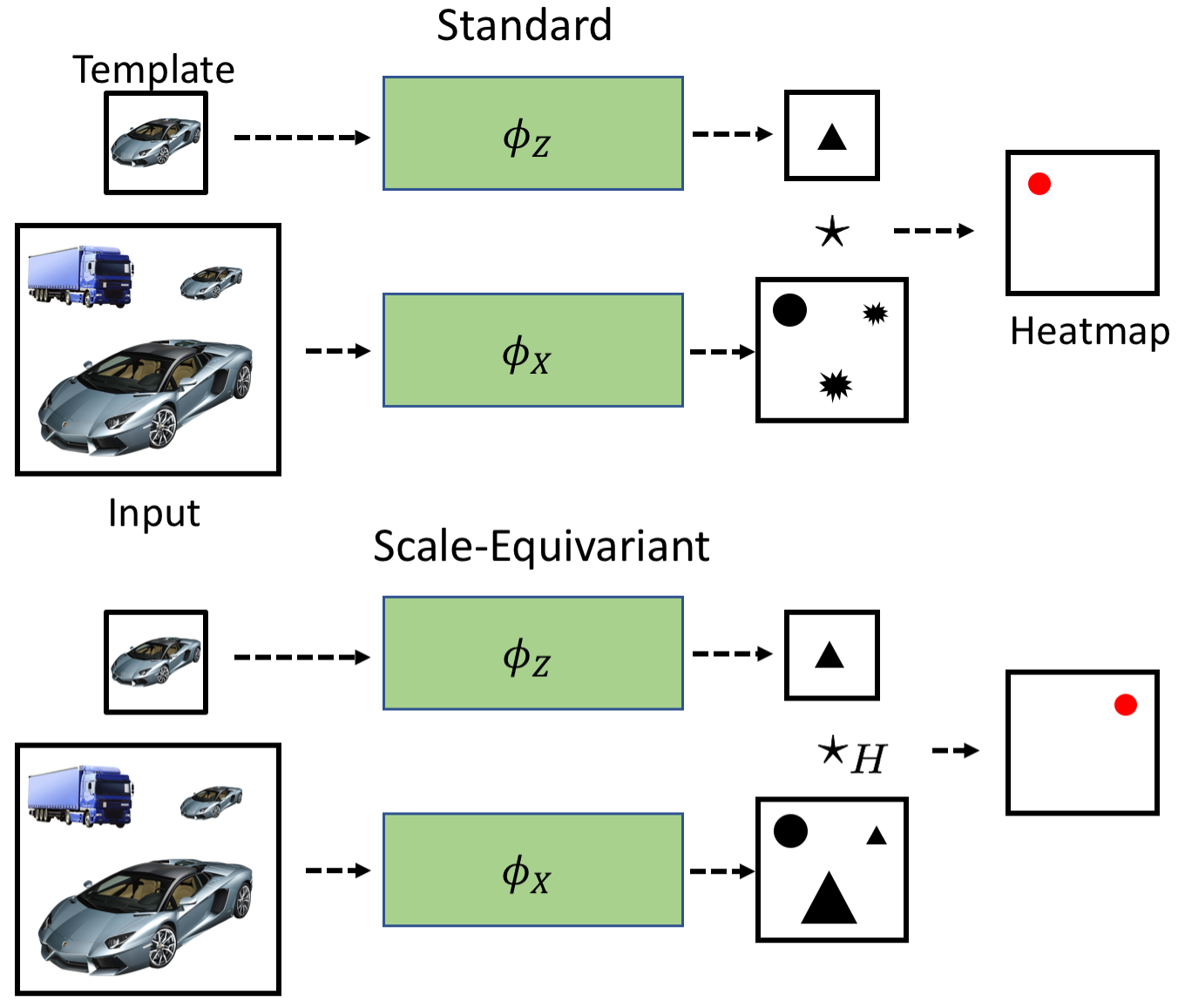}
  \caption{The standard version (top) and the scale-equivariant version (bottom) of a basic tracker. The scale-equivariant tracker has an internal notion of scale which allows for the distinction between similar objects which only differ in scale. The operator $\star$ denotes convolution, and $\star_{H}$ stands for scale-convolution.}
  \label{fig:fig1intro}
\end{figure}
\else
\begin{figure}[t]
  \centering
    \includegraphics[width=0.40\textwidth]{pics/scale_eq_trackers.png}
  \caption{The standard version (top) and the scale-equivariant version (bottom) of a basic tracker. The scale-equivariant tracker has an internal notion of scale which allows for the distinction between similar objects which only differ in scale. The operator $\star$ denotes convolution, and $\star_{H}$ stands for scale-convolution.}
  \label{fig:fig1intro}
\end{figure}
\fi

Measuring scale precisely is crucial when the camera zooms its lens or when the target moves into depth. However, scale is also important in distinguishing among objects in general. In following a marching band or in analyzing a soccer game, or when many objects in the video have a similar appearance (a crowd, team sports), the similarity power of Siamese trackers has a hard time locating the right target. In such circumstances, spatial-scale equivariance will provide a richer and hence more discriminative descriptor, which is essential to differentiate among several similar candidates in an image. And, even, as we will demonstrate, when the sequence does not show variation over scale, proper scale measurement is important to keep the target bounding box stable in size.

The common way to implement scale into a tracker is to train the network on a large dataset where scale variations occur naturally. However, as was noted in \cite{laptev2016ti}, such training procedures may lead to learning groups of re-scaled duplicates of almost the same filters. As a consequence, inter-scale similarity estimation becomes unreliable, see Figure \ref{fig:fig1intro} top. Scale-equivariant models have an internal notion of scale and built-in weight sharing among different filter scales. Thus, scale equivariance aims to produce the same distinction for all sizes, see Figure \ref{fig:fig1intro} bottom.

In this paper, we aim to equip the Siamese network with spatial and scale equivariance built-in from the start to capture the natural variations of the target \textit{a priori}. We aim to improve a broad class of tracking algorithms by enhancing their capacity of candidate distinction. We adopt recent advances \cite{sosnovik2019scale} in convolutional neural networks (CNNs) which handle scale variations explicitly and efficiently.

While scale-equivariant convolutional models have lead to success in image classification \cite{sosnovik2019scale,worrall2017harmonic}, we focus on their usefulness in \textit{object localization}. Where scale estimation has been used in the localization for tracking, it typically relies on brute-force multi-scale detection with an obvious computational burden \cite{danelljan2016beyond, bertinetto2016fully}, or on a separate network to estimate the scale \cite{danelljan2019atom, li2018high}. Both approaches will require attention to avoid bias and the propagation thereof through the network. Our new method treats scale and scale equivariance as a desirable fundamental property, which makes the algorithm conceptually easier. Hence, scale equivariance should be easy to merge into an existing network for tracking. Then, scale equivariance will enhance the performance of the tracker without further modification of the network or extensive data augmentation during the learning phase.

We make the following contributions:
\begin{itemize}
    \item We propose the theory for scale-equivariant Siamese trackers and provide a simple recipe of how to make a wide range of existing trackers scale-equivariant.
    \item We propose building blocks necessary for efficient implementation of scale equivariance into modern Siamese trackers and implement a scale-equivariant extension of the recent SiamFC+ \cite{zhang2019deeper} tracker.
    \item We demonstrate the advantage of scale-equivariant Siamese trackers over their conventional counterparts on popular benchmarks for sequences with and without apparent scale changes.

\end{itemize}
\section{Related Work}
\label{sec:related_work}

\begin{figure*}[t]
    \begin{center}
        \includegraphics[width=0.9\linewidth]{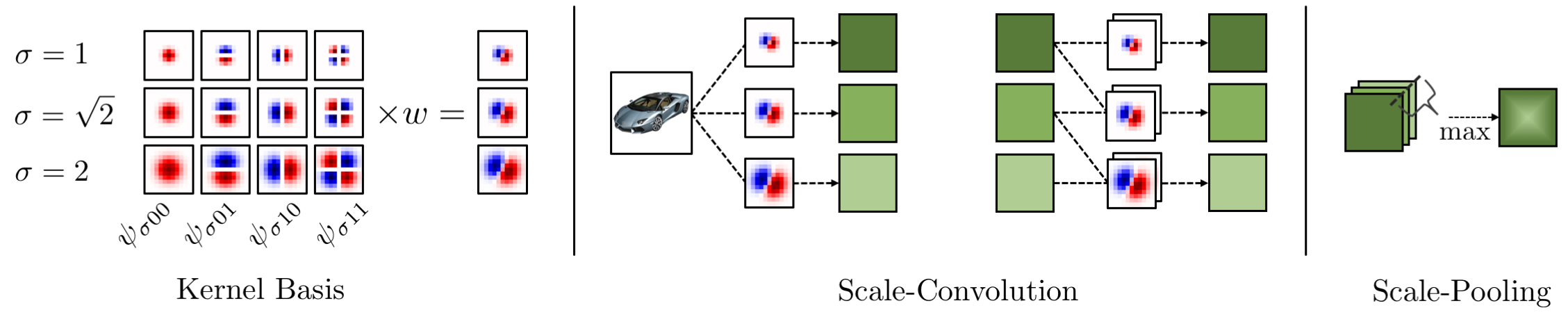}
    \end{center}
    \caption{Left: convolutional kernels use a fixed kernel basis on multiple scales, each with a set of trainable weights. Middle: a representation of scale-convolution using \Eqref{eq:scale_conv_def} for the first and all subsequent layers. Right: a scheme of scale-pooling, which transforms a 3D-signal into a 2D one without losing scale equivariance. As an example, we use a basis of 4 functions and 3 scales with a step of $\sqrt{2}$. Only one channel of each convolutional layer is demonstrated for simplicity.}
    \label{fig:scale_modules}
\end{figure*}

\paragraph{Siamese tracking} 
The challenge of learning to track arbitrary objects can be addressed by deep similarity learning \cite{bertinetto2016fully}. The common approach is to employ Siamese networks to compute the embeddings of the original patches. The embeddings are then fused to obtain a location estimate. Such formulation is general, allowing for a favourable flexibility in the design of the tracker. 
In \cite{bertinetto2016fully} Bertinetto \etal employ off-line trained CNNs as feature extractors. The authors compare dot-product similarities between the feature map of the template with the maps coming from the current frame and measure similarities on multiple scales. Held \etal \cite{held2016learning} suggest a detection-based Siamese tracker, where the similarity function is modeled as a fully-connected network. Extensive data augmentation is applied to learn a similarity function, which generalizes for various object transformations.
Li \etal \cite{li2018high} consider tracking as a one-shot detection problem to design Siamese region-proposal-networks \cite{renNIPS15fasterrcnn} by fusing the features from a fully-convolutional backbone. The recent ATOM \cite{danelljan2019atom} and DIMP \cite{bhat2019dimp} trackers employ a multi-stage tracking framework, where an object is coarsely localized by the online \textit{classification} branch, and subsequently refined in its position by the \textit{estimation} branch. From a Siamese perspective, in both \cite{danelljan2019atom, bhat2019dimp} the object embeddings are first fused to produce an initial location and subsequently  processed by the IoU-Net \cite{jiang2018iounet} to enhance the precision of the bounding box.

The aforementioned references have laid the foundation for most of the state-of-the-art trackers. These methods share an implicit or explicit attention to translation equivariance for feature extraction. The decisive role of translation equivariance is noted in \cite{bertinetto2016fully, li2018siamrpn++, zhang2019deeper}. Bertinetto \etal \cite{bertinetto2016fully} utilize fully-convolutional networks where the output directly commutes with a shift in the input image as a function of the total stride. Li \etal \cite{li2018siamrpn++} suggest a training strategy to eliminate the spatial bias introduced in non-fully-convolutional backbones. Along the same line, Zhang and Peng \cite{zhang2019deeper} demonstrated that deep state-of-the-art models developed for classification are not directly applicable for localization. And hence these models are not directly applicable to tracking as they induce positional bias, which breaks strict translation equivariance. We argue that transformations, other then translation, such as rotation may be equally important for certain classes of videos like sports and following objects in the sea or in the sky. 
And we argue that scale transformation is common in the majority of sequences due to the changing distances between objects and the camera. In this paper, we take on the latter class of transformations for tracking.
\vspace{-3mm}
\paragraph{Equivariant CNNs}
Various works on transformation-equivariant convolutional networks have been published recently. They extend the built-in property of translation-equivariance of conventional CNNs to a broader set of transformations. Mostly considered was roto-translation, as demonstrated on image classification  \cite{cohen2016group,cohen2016steerable,hoogeboom2018hexaconv,weiler2019general,romero2019co,romero2020attentive}, image segmentation \cite{weiler2018learning} and edge detection \cite{worrall2017harmonic}.

One of the first works on scale-translation-equivariant convolutional networks was by Marcos \etal \cite{marcos2018scale}. In order to process images on multiple scales, the authors resize and convolve the input of each layer multiple times, forming a stack of features which corresponds to variety of scales. The output of such a convolutional layer is a vector whose length encodes the maximum response in each position among different scales. The direction of the vector is derived from the scale, which gave the maximum. The method has almost no restrictions in the choice of admissible scales. As this approach relies on rescaling the image, the obtained models are significantly slower compared to conventional CNNs.
Thus, this approach is not suitable for being applied effectively in visual object tracking.

Worrall \& Welling \cite{worrall2019deep} propose Deep Scale-Spaces, an equivariant model which generalizes the concept of scale-space to deep networks. The approach uses filter dilation to analyze the images on different scales. It is almost as fast as a conventional CNN with the same width and depth. As the method is restricted to integer scale factors it is unsuited to applications in tracking where the scene dictates arbitrary scale factors. 

Almost simultaneously, three papers \cite{sosnovik2019scale,bekkers2019b,zhu2019scale} were proposed to implement scale-translation-equivariant networks with arbitrary scales. What they have in common is that they use a pre-calculated and fixed basis defined on multiple scales. All filters are then calculated as a linear combination of the basis and trainable weights. As a result, no rescaling is used. We prefer to use \cite{sosnovik2019scale}, as Sosnovik \etal propose an approach for building general scale-translation-equivariant networks with an algorithm for the fast implementation of the scale-convolution.

To date, the application of scale-equivariant networks was mostly demonstrated in image classification. Almost no attention was paid to tasks that involve object localization, such as visual object tracking. As we have noted above, it is a fundamentally different case. To the best of our knowledge, we demonstrate the first application of transformation-equivariant CNNs to visual object tracking.

\section{Scale-Equivariant Tracking}
\label{sec:se_tracking}
In this work, we consider a wide range of modern trackers which can be described by the following formula:
\begin{equation}
    \label{eq:tracker_def}
    h(z, x) = \phi_X(x) \star \phi_Z(z)
\end{equation}
where $z, x$ are the template and the input frame, and $\phi_X, \phi_Z$ are the functions which process them, and $\star$ is the convolution operator which implements a connection between two signals. The resulting value $h(z, x)$ is a heatmap that can be converted into a prediction by relatively simple calculations. Functions $\phi_X, \phi_Z$ here can be parametrized as feed-forward neural networks. For our analysis, it is both suitable if the weights of these networks are fixed or updated during training or inference. 
This pipeline describes the majority of Siamese trackers such as \cite{bertinetto2016fully,li2018high,li2018siamrpn++} and the trackers based on correlation filters \cite{danelljan2017eco,danelljan2016beyond}.

\subsection{Convolution is all you need}
Let us consider some mapping $g$. It is equivariant under a transformation $L$ if and only if there exists $L'$ such that $g \circ L = L' \circ g$. If $L'$ is the identity mapping, then the function $g$ is invariant under this transformation. A function of multiple variables is equivariant when it is equivariant with respect to each of the variables. In our analysis, we consider only transformations that form a transformation group, in other words, $L \in G$.

\begin{theorem}
\label{theorem:main}
A function given by \Eqref{eq:tracker_def} is equivariant under a transformation $L$ from group $G$  if and only if 
$\phi_X$ and $\phi_Z$ are constructed from $G$-equivariant convolutional layers and $\star$ is the $G$-convolution.
\end{theorem}

The proof of Theorem \ref{theorem:main} is given in the supplementary material. A simple interpretation of this theorem is that \textit{a tracker is equivariant to transformations from $G$ if and only if it is fully $G$-convolutional}. The necessity of fully-convolutional trackers is well-known in tracking community and is related to the ability of the tracker to capture the main variations in the video --- the translation. In this paper, we seek to extend this ability to scale variations as well. Which, due to Theorem \ref{theorem:main} boils down to using scale-convolution and building fully scale-translation convolutional trackers.

\subsection{Scale Modules}
Given a function $f: \mathbb{R} \rightarrow \mathbb{R}$, a scale transformation is defined as follows:
\begin{equation}
    \label{eq:scale_def}
    L_s[f](t) = f(s^{-1}t), \quad \forall s \geq 0
\end{equation}
where cases with $s > 1$ are referred to as upscale and with $s < 1$ as downscale. Standard convolutional layers and convolutional networks are translation equivariant but not scale-equivariant \cite{sosnovik2019scale}. 
\paragraph{Parametric Scale-Convolution}
In order to build scale-equivariant convolutional networks, we follow the method proposed by Sosnovik \etal \cite{sosnovik2019scale}. We begin by choosing a complete basis of functions defined on multiple scales. Choosing the center of the function to be the point $(0, 0)$ in coordinates $(u, v)$, we use functions of the following form:
\begin{equation}
    \label{eq:steerable_funcs}
    \psi_{\sigma nm}(u, v) = A\frac{1}{\sigma^2}
    H_n\Big(\frac{u}{\sigma}\Big)
    H_m\Big(\frac{v}{\sigma}\Big)
    e^{ -\frac{u^2 + v^2}{2\sigma^2}}
\end{equation}
Here $H_n$ is a Hermite polynomial of the $n$-th 
order, and $A$ is a constant used for normalization. 
In order to build a basis of $N$ functions, we iterate over increasing pairs of $n$ and $m$. As the basis is complete, the number of functions $N$ is equal to the number of pixels in the original filter. We build such a basis for a chosen set of equidistant scales $\sigma$ and fix it:
\begin{equation}
    \label{eq:steerable_basis}
    \Psi_{\sigma } = \Big\{ 
    \psi_{\sigma 0 0},\;
    \psi_{\sigma 0 1},\;  
    \psi_{\sigma 1 0},\; 
    \psi_{\sigma 1 1} \dots\Big\}
\end{equation}

Kernels of convolutional layers are parametrized by trainable weights $w$ in the following way:
\begin{equation}
    \label{eq:scale_kernels}
    \kappa_{\sigma} = \sum_i \Psi_{\sigma i} w_i 
\end{equation}
As a result, each kernel is defined on multiple scales and no image interpolation is used.
Given a function of scale and translation $f(s, t)$ and a kernel  $\kappa_\sigma(s, t)$, a scale convolution is defined as:
\begin{equation}
    \label{eq:scale_conv_def}
    [f \star_H \kappa_\sigma](s, t)
    = \sum_{s'}  [f(s', \cdot) \star \kappa_{s\cdot\sigma}(s^{-1}s', \cdot)](t)
\end{equation}
The result of this operation is a stack of features each of which corresponds to a different scale. We end up with a 3-dimensional representation of the signal --- 2-dimensional translation $+$ scale.
We follow \cite{sosnovik2019scale} and denote scale-convolution as $\star_H$ in order to distinguish it with the standard one. \Figref{fig:scale_modules} demonstrates how a kernel basis is formed and how scale-convolutional layers work.

\vspace{-1mm}
\paragraph{Fast $\bm{1 \times 1}$ Scale-Convolution} \label{sec:fast11}
An essential building block of many backbone deep networks such as ResNets \cite{he2016deep} and Wide ResNets \cite{zagoruyko2016wide} is a $1 \times 1$ convolutional layer. We follow the interpretation of these layers proposed in \cite{lin2013network} --- it is a linear combination of channels. Thus, it has no spatial resolution. In order to build a scale-equivariant counterpart of $1 \times 1$ convolution, we do not utilize a kernel basis. As we pointed out before, the signal is stored as a 3 dimensional tensor for each channel. Therefore, for a kernel defined on $N_S$ scales, the convolution of the signal with this kernel is just a 3-dimensional convolution with a kernel of size $1\times1$ in spatial dimension, and with $N_S$ values in depth. This approach for $1 \times 1$ scale-convolution is faster than the special case of the algorithm proposed in \cite{sosnovik2019scale}.

\vspace{-1mm}
\paragraph{Padding}
Although zero padding is a standard approach in image classification for saving the spatial resolution of the image, it worsens the localization properties of convolutional trackers \cite{li2018siamrpn++, zhang2019deeper}. Nevertheless, a simple replacement of standard convolutional layers with scale-equivariant ones in very deep models is not possible without padding. Scale-equivariant convolutional layers have kernels of a bigger spatial extent because they are defined on multiple scales. For these reasons, we use circular padding during training and zero padding during testing in our models.

The introduced padding does not affect the feature maps which are obtained with kernels defined on small scales. It does not violate the translation equivariance of a network. We provide an experimental proof in supplementary material.

\vspace{-1mm}
\paragraph{Scale-Pooling}
In order to capture correlations between different scales and to transform a 3-dimensional signal into a 2-dimensional one, we utilize global max pooling along the scale axis. This operation does not eliminate the scale-equivariant properties of the network. We found that it is useful to additionally incorporate this module in the places where conventional CNNs have spatial max pooling or strides. The mechanism of scale-pooling is illustrated in \Figref{fig:scale_modules}.

\vspace{-1mm}
\paragraph{Non-parametric Scale-Convolution}
The convolutional operation which results in the heatmap of a tracker is non-parametric. Both the input and the kernel come from neural networks. Thus, the approach described in \Eqref{eq:scale_conv_def} is not suitable for this case. Given two functions $f_1, f_2$ of scale and translation the non-parametric scale convolution is defined as follows:
\begin{equation}
    \label{eq:scale_conv_nonparam}
    [f_1 \star_H f_2](s, t) = L_{s^{-1}}[L_s[f_1] \star f_2](t)
\end{equation}
Here $L_s$ is rescaling implemented as bicubic interpolation. Although it is a relatively slow operation, it is used only once in the tracker and does not heavily affect the inference time. The proof of the equivariance of this convolution is provided in supplementary material.

\subsection{Extending a Tracker to Scale Equivariance} 
\label{sec:Recipe}
We present a recipe to extend a tracker to scale equivariance.
\begin{enumerate}
    \item The first step is to estimate to what degree objects change in size in this domain, and then to select a set of scales $\sigma_1, \sigma_2, \dots \sigma_N$. This is a domain-specific hyperparameter.
    For example, a domain with significant scale variations requires a broader span of scales, while for more smooth sequences, the set may consist of just 3 scales around 1.
    \item For a tracker which can be described by \Eqref{eq:tracker_def}, derive $\phi_X$ and $\phi_Z$.
    \item For the networks represented by $\phi_X$ and $\phi_Z$, all convolutional layers need to be replaced with scale-convolutional layers. The basis for these layers is based on the chosen scales  $\sigma_1, \sigma_2, \dots \sigma_N$. 
    \item (Optional) Scale-pooling can be included to additionally capture inter-scale correlations between all scales.
    \item The connection operation $\star$ needs to be replaced with a non-parametric scale-convolution.
    \item (Optional) If the tracker only searches over spatial locations, scale-pooling needs to be included at the very end.
\end{enumerate}

The obtained tracker produces a heatmap $h(z, x)$ defined on scale and translation. Therefore, each position is assigned a vector of features that has both the measure of similarity and the scale relation between the candidate and the template. If additional scale-pooling is included, then all scale information is just aggregated in the similarity score.

Note that the overall structure of the tracker, as well as the training and inference procedures are not changed. 
Thus, the recipe allows for a simple extension of a tracker with little cost of modification.

\section{Scale-Equivariant SiamFC}

\label{sec:models}
While the proposed algorithm is applicable to a wide range of trackers, in this work, we focus on Siamese trackers. As a baseline we choose SiamFC \cite{bertinetto2016fully}. This model serves as a starting point for modifications for the many modern high-performance Siamese trackers.

\subsection{Architecture}
Given the recipe, here we discuss the actual implementation of the scale-equivariant SiamFC tracker (SE-SiamFC).
 
In the first step of the recipe, we assess the range of scales in the domain (dataset). In sequences presented in most of the tracking benchmarks, like OTB or VOT, objects change their size relatively slowly from one frame to the other. The maximum scale change usually does not exceed a factor of $1.5-2$. Therefore, we use $3$ scales with a step of $\sqrt{2}$ as the basis for the scale-convolutions. 
The next step in the recipe is to represent the tracker as it is done in \Eqref{eq:tracker_def}. SiamFC localizes the object as the coordinate \textit{argmax} of the heatmap $h(z, x) = \phi_Z(z) \star \phi_X(x)$, where $\phi_Z=\phi_X$ are convolutional Siamese backbones. 
Next, in step number 3, we modify the backbones by replacing standard convolutions by scale-equivariant convolutions. We follow step 4 and utilize scale-pooling in the backbones in order to capture additional scale correlations between features of various scales. According to step 5, the connecting correlation is replaced with non-parametric scale-convolution. SiamFC computes its similarity function as a 2-dimensional map, therefore, we follow step 6 and add extra scale-pooling in order to transform a 3-dimensional heatmap into a 2-dimensional one. Now, we can use exactly the same inference algorithm as in the original paper \cite{bertinetto2016fully}. We use the standard approach of scale estimation, based on the greedy selection of the best similarity for 3 different scales.

\begin{figure*}[t!]
  \centering
    \includegraphics[width=\textwidth]{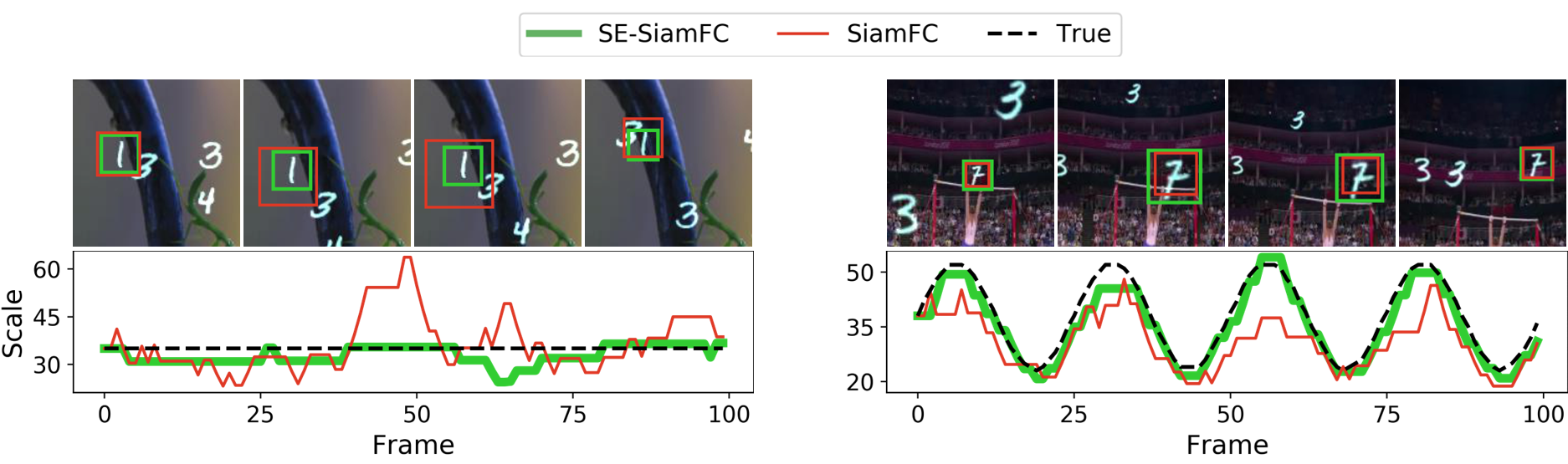}
  \caption{Top: examples of simulated T-MNIST and S-MNIST sequences. Bottom: scale estimation for equivariant and non-equivariant models. In the S-MNIST example, SE-SiamFC can estimate the scale more accurately. In the T-MNIST example, our model better preserves the scale of the target unchanged, while the non-scale-equivariant model is prone to oscillations in its scale estimate.}
  \label{fig:mnist}
\end{figure*}

\vspace{3mm}
\subsection{Weight Initialization}
\label{sec:weight_initialization}
An important ingredient of a successful model training is the initialization of its weights. A common approach is to use weights from an Imagenet \cite{deng2009imagenet} pre-trained model \cite{li2018high,zhang2019deeper,li2018siamrpn++}. In our case, however, this requires additional steps, as there are no available scale-equivariant models pre-trained on the Imagenet. 
We present a method for initializing a scale-equivariant model with weights from a pre-trained conventional CNN. The key idea is that a scale-equivariant network built according to \Secref{sec:Recipe} contains a sub-network that is identical to the one of the non-scale-equivariant counterpart. As the kernels of scale-equivariant models are parameterized with a fixed basis and trainable weights, our task is to initialize these weights.

We begin by initializing the inter-scale correlations by setting to $0$ all weights responsible for these connections. At this moment, up to scale-pooling, the scale-equivariant model consists of several networks parallel to, yet disconnected from one another, where the only difference is the size of their filters. For the convolutional layers with a non-unitary spatial extent, we initialize the weights such that the kernels of the smallest scale match those of the source model. Given a source kernel $\kappa'(u, v)$ and a basis $\Psi_{\sigma i}(u, v)$ with $\sigma=1$, weights $w_i$ are chosen to satisfy the linear system derived from \Eqref{eq:scale_kernels}:
\begin{equation}
    \label{eq:weight_init_1}
    \kappa_1(u, v) = \sum_i \Psi_{1 i}(u, v) w_i = \kappa'(u, v),
    \quad \forall u,v
\end{equation}
As the basis is complete by construction, its matrix form is invertible. The system has a unique solution with respect to $w_i$:
\begin{equation}
    w_i = \sum_{u,v}\Psi^{-1}_{1i}(u, v) \kappa'(u, v)
\end{equation}
All $1\times1$ scale-convolutional layers are identical to standard $1\times1$ convolutions after zeroing out inter-scale correlations. We copy these weights from the source model. We provide an additional illustration of the proposed initialization method in the supplementary material.

\begin{table}[t!]
\begin{center}
\begin{tabular}{@{}lllllc@{}}
\toprule
Tracker & T/T & T/S & S/T & S/S & \multicolumn{1}{l}{\# Params} \\ \midrule
SiamFC & \multicolumn{1}{c}{0.64} & \multicolumn{1}{c}{0.62} & \multicolumn{1}{c}{0.64} & \multicolumn{1}{c}{0.63} & 999 K \\
SE-SiamFC & \textbf{0.76} & \textbf{0.69} & \textbf{0.77} & \textbf{0.70} & 999 K \\ \bottomrule
\end{tabular}
\caption{AUC for models trained on T-MNIST and S-MNIST. T/S indicates that the model was trained on T-MNIST and tested on S-MNIST datasets. Bold numbers represent the best result for each of the training/testing scenarios.}
\label{tab:mnist}
\end{center}
\end{table}
\section{Experiments and Results}
\label{sec:experiments} 

\begin{table*}[h]
\begin{center}
\begin{tabular}{lclcclcclccclccc}
\toprule
\multirow{2}{*}{Tracker} & \multirow{2}{*}{Year} &  & \multicolumn{2}{c}{OTB-2013} &  & \multicolumn{2}{c}{OTB-2015} &  & \multicolumn{3}{c}{VOT2016} &  & \multicolumn{3}{c}{VOT2017} \\
\cmidrule{4-5} \cmidrule{7-8} \cmidrule{10-12} \cmidrule{14-16} 
&&& AUC & Prec. && AUC & Prec. && EAO & A & R && EAO & A & R \\ 
\midrule
SINT \cite{tao2016siamese} & 2016 && 0.64 & 0.85 && - & - && - & - & - && - & - & - \\
SiamFC \cite{bertinetto2016fully} & 2016 && 0.61 & 0.81 && 0.58 & 0.77 && 0.24 & 0.53 & 0.46 && 0.19 & 0.50 & 0.59 \\
DSiam \cite{guo2017learning} & 2017 && 0.64 & 0.81 && - & - && - & - & - && - & - & -\\ 
StructSiam \cite{zhang2018structured} & 2018 && 0.64 & 0.88 && 0.62 & 0.85 && 0.26 & - & - && - & - & - \\
TriSiam \cite{dong2018triplet} & 2018 && 0.62 & 0.82 && 0.59 & 0.78 && - & - & - && 0.20 & - & - \\
SiamRPN \cite{li2018high} & 2018 && - & - && 0.64 & 0.85 && 0.34 & 0.56 & 0.26 && 0.24 & 0.49 & 0.46 \\
SiamFC+ \cite{zhang2019deeper} & 2019 && 0.67 & 0.88 && 0.64 & 0.85 && 0.30 & 0.54 & 0.38 && 0.23 & 0.50 & 0.49 \\
\midrule 
SE-SiamFC & Ours && \textbf{0.68} & 0.90 && \textbf{0.66} & 0.88 && \textbf{0.36} & 0.59 & 0.24 && \textbf{0.27} & 0.54 & 0.38 \\
\bottomrule
\end{tabular}
\caption{Performance comparisons on OTB-2013, OTB-2015, VOT2016, and VOT2017 benchmarks. Bold numbers represent the best result for each of the benchmarks.}
\label{tab:results}
\end{center}
\end{table*}

\subsection{Translation-Scaling MNIST}
\label{sec:experiments_mnist}

To test the ability of a tracker to cope with translation and scaling, we conduct an experiment on a simulated dataset with controlled factors of variation. We construct the datasets of translating (T-MNIST) and translating-scaling (S-MNIST) digits.

In particular, to form a sequence, we randomly sample up to $8$ MNIST digits with backgrounds from the GOT10k dataset \cite{huang2019got10k}. Then, on each of the digits in the sequence independently, a smoothed Brownian motion model induces a random translation. Simultaneously, for S-MNIST, a smooth scale change in the range $[0.67, 1.5]$ is induced by the sine rule:
\begin{equation}
    s_i(t) = \frac{h-l}{2} \big{[}\sin(\frac{t}{4} + \beta_i) + 1)\big{]} + l
\end{equation}
where $s_i(t)$ is the scale factor of the $i$-th digit in the $t$-th frame, $h,l$ are upper and lower bounds for scaling, and $\beta_i \in [0, 100]$ is a phase, sampled randomly for each of the digits. In total, we simulate 1000 sequences for training and 100 for validation. Each sequence has a length of $100$ frames.
We compare two configurations of the tracker: (i) SiamFC with a shallow backbone and (ii) its scale-equivariant version SE-SiamFC. We conduct the experiments according to $2 \times 2$ scenarios: the models are trained on either S-MNIST or T-MNIST and are subsequently tested on either of them. The results are listed in Table \ref{tab:mnist}. See supplementary material for a detailed description of the architecture, training, and testing procedures.

As can be seen from Table \ref{tab:mnist}, the equivariant version outperforms its non-equivariant counterpart in all scenarios. The experiment on S-MNIST, varying the scale of an artificial object, shows that the scale-equivariant model has a superior ability to precisely follow the change in scale compared to the conventional one. The experiment on T-MNIST shows that (proper) measurement of scale is important even in the case when the sequence does not show a change in scale, where the observed scale in SE-SiamFC fluctuates much less than it does in the baseline (see Figure \ref{fig:mnist}).

\subsection{Benchmarking}\label{sec:benchmarking}
We compare the scale-equivariant tracker against a non-equivariant baseline on popular tracking benchmarks. We test SE-SiamFC with a backbone from \cite{zhang2019deeper} against other popular Siamese trackers on OTB-2013, OTB-2015, VOT2016, and VOT2017. The benchmarks are chosen to allow direct comparison with the baseline \cite{zhang2019deeper}. 
We compare against the results published in the original paper. Although additional results are presented online\footnote{\url{https://github.com/researchmm/SiamDW}}, we couldn't reproduce them in a reasonable amount of time.
\vspace{-3mm}
\paragraph{Implementation details} The parameters of our model are initialized with weights pre-trained on Imagenet by a method described in Section \ref{sec:weight_initialization}. We use the same training procedure as in the baseline. See supplementary material for a detailed description of the architecture.

The pairs for training are collected from the GOT10k \cite{huang2019got10k} dataset. We adopt the same prepossessing and augmentation techniques as in \cite{zhang2019deeper}.  The inference procedure remains unchanged compared to the baseline.

\begin{figure}[!t]
    \begin{center}
        \includegraphics[width=0.98\linewidth]{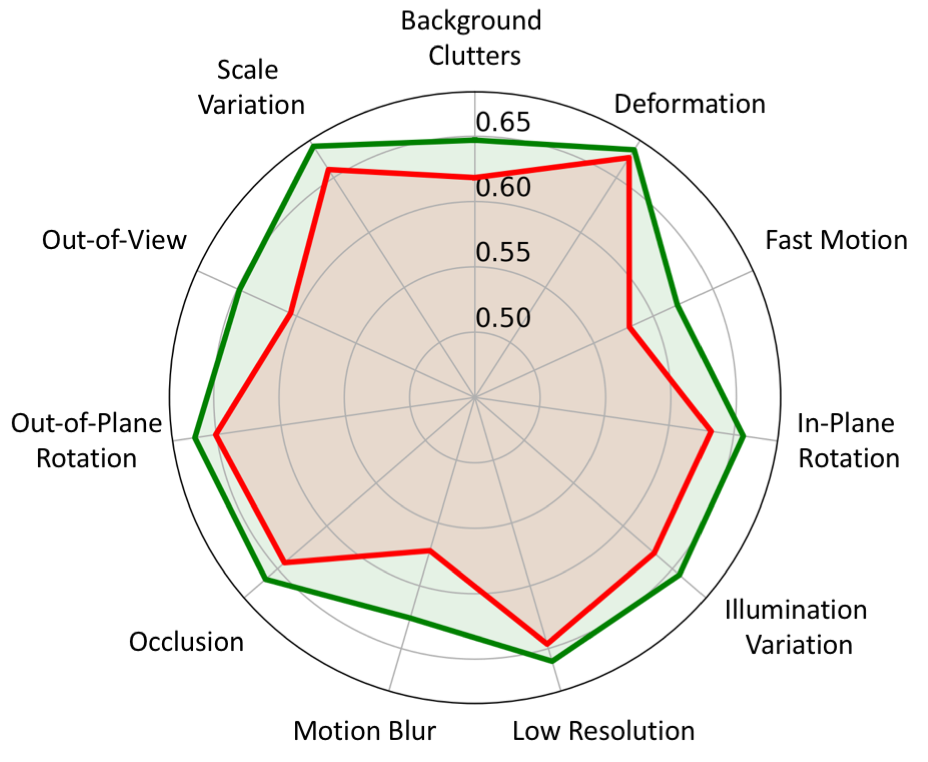}
    \end{center}
    \caption{Comparison of AUC on OTB-2013 with different factors of variations. The red polygon corresponds to the baseline SiamFC+ and the green polygon --- to SE-SiamFC.}
    \label{fig:rpg}
\end{figure}

\vspace{-3mm}
\paragraph{OTB} 

We test on the OTB-2013 \cite{otb13} and OTB-2015 \cite{otb15} benchmarks. Each of the sequences in the OTB datasets carries labels from 11 categories of difficulty in tracking the sequence. Examples of these labels include: occlusion, scale variation, in-pane rotation, \textit{etc.}  We employ a standard one-pass evaluation (OPE) protocol to compare our method with other trackers by the area under the success curve (AUC) and precision. 

The results are reported in Table \ref{tab:results}. Our scale-equivariant tracker outperforms its non-equivariant counterpart by more than $3\%$ on OTB-2015 in both AUC and precision, and by $1.4\%$ on OTB-2013. When summarized at each label of difficulty (see Figure \ref{fig:rpg}), the proposed scale-equivariant tracker is seen to improve all sequence types, not only those labeled with ``scale variation''. 

We attribute this to the fact that the ``scale variation'' tag in the OTB benchmark only indicates the sequences with a relatively big change in scale factors, while up to a certain degree, scaling is present in almost any video sequence. Moreover, scaling may be present implicitly, in the form of the same patterns being observed on multiple scales. An ability of our model to exploit this leads to better utilization of trainable parameters and a more discriminative Siamese similarity as a result.

\vspace{-3mm}
\paragraph{VOT} We next evaluate our tracker on VOT2016 and VOT2017 datasets \cite{vottoolkit}. The performance is evaluated in terms of average bounding box overlap ratio (A), and the robustness (R). These two metrics are combined into the Expected Average Overlap (EAO), which is used to rank the overall performance.

The results are reported in Table \ref{tab:results}. On VOT2016 our scale-equivariant model shows an improvement from $0.30$ to $0.36$ in terms of EAO, which is a $20\%$ gain compared to the non-equivariant baseline. On VOT2017, the increase in EAO is $17\%$.

We qualitatively investigated the sequences with the largest performance gain and observed that the most challenging factor for our baseline is the rapid scaling of the object. Even when the target is not completely lost, the imprecise bounding box heavily influences the overlap with the ground truth and the final EAO. Our scale-equivariant model better adapts to the fast scaling and delivers tighter bounding boxes. We provide qualitative results in the supplementary material. 

\begin{table}[t!]
\centering
\begin{tabular}{@{}lcc@{}}
\toprule
\multicolumn{1}{c}{Model} & Pretrained weigths & AUC \\ \midrule
SiamFC+ (aug. 5\%) & \cmark &0.668 \\
SiamFC+ (aug. 20\%) & \cmark &0.668 \\
SiamFC+ (aug. 50\%) & \cmark &0.664 \\
\midrule
SE-SiamFC $\sigma=1.2$ & \cmark &0.677 \\
SE-SiamFC $\sigma=1.3$ & \cmark &0.680 \\
SE-SiamFC $\sigma=1.4$ & \cmark &\textbf{0.681} \\
SE-SiamFC $\sigma=1.5$ & \cmark &0.678 \\
\midrule
SE-SiamFC $\sigma=1.4$ & \xmark &0.553 \\
\bottomrule
\end{tabular}
\caption{Ablation study on the OTB-2013 benchmark. The parameter $\sigma$ stands for the step between scales in scale-equivariant models. Bold numbers represent the best result.}
\label{tab:ablation}
\end{table}

\subsection{Ablation Study}

We conduct an ablation study on the OTB-2013 benchmark to investigate the impact of scale step, weight initialization, and fast $1\times 1$ scale-convolution. We also test the baseline SiamFC+ model with various levels of scale data augmentation during the training. We follow the same training and testing procedure as in Section \ref{sec:benchmarking} for all experiments. In the weight initialization experiment, however, we do not use gradual weights unfreezing, but train the whole model end-to-end from the first epoch.
\vspace{-3mm}
\paragraph{Scale step} We investigate the impact of scale step $\sigma$, which defines a set of scales our model operates on. We train and test SE-SiamFC with various scale steps. Results are shown in Table \ref{tab:ablation}. It can be seen that the resulting method outperforms the baseline on a range of scale steps. We empirically found that $\sigma=1.4$ achieves the best performance.
\vspace{-3mm}
\paragraph{Scale data augmentation} Data augmentation is a common way to improve model generalization over different variations. Since our method is focused on scale, we compare SE-SiamFC against a baseline trained with different levels of scale data augmentation. Our results indicate (Table \ref{tab:ablation}) that scale augmentation does not improve the performance of the conventional non-equivariant tracker.
\vspace{-4mm}
\paragraph{Weight Initialization} We train and test SE-SiamFC model, where weights initialized randomly \cite{glorot2010init,weiler2018learning}. As can be seen from Table \ref{tab:ablation}, random initialization results in a \textit{19\%} performance drop compared to the proposed initialization technique.
\vspace{-4mm}
\paragraph{Fast $1\times1$ scale-convolution} We compare the speed of $1\times1$ scale-convolution from \cite{sosnovik2019scale} and the proposed fast implementation. Implementation from \cite{sosnovik2019scale} requires $\text{450 / 1650 }\mu\text{s}$, while our implementation requires $\text{67 / 750 }\mu\text{s}$ for forward / backward pass respectively, which is more than 6 times faster. In our experiments, the usage of fast $1\times1$ scale-convolution results in $30-40\%$ speedup of a tracker.

\section{Discussion}
\label{sec:conclusion}

In this work, we argue about the usefulness of additional scale equivariance in visual object tracking for the purpose of enhancing Siamese similarity estimation. We present a general theory that applies to a wide range of modern Siamese trackers, as well as all the components to turn an existing tracker into a scale-equivariant version. Moreover, we prove that the presented components are both necessary and sufficient to achieve built-in scale-translation equivariance. We sum up the theory by developing a simple recipe for extending existing trackers to scale equivariance. We apply it to develop SE-SiamFC --- a scale-equivariant modification of the popular SiamFC tracker.

We experimentally demonstrate that our scale-equivariant tracker outperforms its conventional counterpart on OTB and VOT benchmarks and on the synthetically generated T-MNIST and S-MNIST datasets, where T-MNIST is designed to keep the object at a constant scale, and S-MNIST varies the scale in a known manner.

The experiments on T-MNIST and S-MNIST show the importance of proper scale measurement for all sequences, regardless of whether they have scale change or not. For the standard OTB and VOT benchmarks, our tracker proves the power of scale equivariance. It is seen to not only improves the tracking in the case of scaling, but also when other factors of variations are present (see \Figref{fig:rpg}). It affects the performance in two ways: it prevents erroneous jumps to similar objects at a different size and it provides a better consistent estimate of the scale.

\subsubsection*{Acknowledgments}
We thank Thomas Andy Keller, Konrad Groh, Zenglin Shi and Deepak Gupta for valuable comments, discussions and help with the project. We appreciate the help of Zhipeng Zhang and Houwen Peng in reproducing the experiments from \cite{zhang2019deeper}.

{\small
\bibliographystyle{ieee_fullname}
\bibliography{bib}
}

\appendix
\clearpage

\section{Proofs}

\subsection{Convolution is all you need}
In the paper we consider trackers of the following form
\begin{equation}
    \label{eq:supp_tracker_def}
    h(z, x) = \phi_X(x) \star \phi_Z(z)
\end{equation}
where $\phi_X$ and $\phi_Z$ are parameterized with feed-forward neural networks.

\begin{theorem}
\label{theorem:main}
A function given by \Eqref{eq:supp_tracker_def} is equivariant under a transformation $L$ from group $G$  if and only if 
$\phi_X$ and $\phi_Z$ are constructed from $G$-equivariant convolutional layers and $\star$ is the $G$-convolution.
\end{theorem}

\begin{proof}
Let us fix $z=z_0$ and introduce a function $h_X = h(x, z_0) =\phi_X(x) \star \phi_Z(z_0)$. This function is a feed-forward neural network. All its layers but the last one are contained in $\phi_X$ and the last layer is a convolution with $\phi_Z(z_0)$. According to \cite{kondor2018generalization} a feed-forward neural network is equivariant under transformations from $G$ if and only if it is constructed from $G$-equivariant convolutional layers. Thus, the function $h_X$ is equivariant under transformations from $G$ if and only if 
\begin{itemize}
    \item The function $\phi_X$ is constructed from $G$-equivariant convolutional layers 
    \item The convolution $\star$ is the $G$-convolution
\end{itemize}
If we then fix $x=x_0$, we can show that a function $h_Z = h(x_0, z) =\phi_X(x_0) \star \phi_Z(z)$ is equivariant under transformations from $G$ if and only if 
\begin{itemize}
    \item The function $\phi_Z$ is constructed from $G$-equivariant convolutional layers 
    \item The convolution $\star$ is the $G$-convolution
\end{itemize}

The function $h$ is equivariant under $G$ if and only if both the function $h_X$ and the function $h_Z$ are equivariant.
\end{proof}

\subsection{Non-parametric scale-convolution}
\begin{figure}[t]
    \begin{center}
        \includegraphics[width=0.99\linewidth]{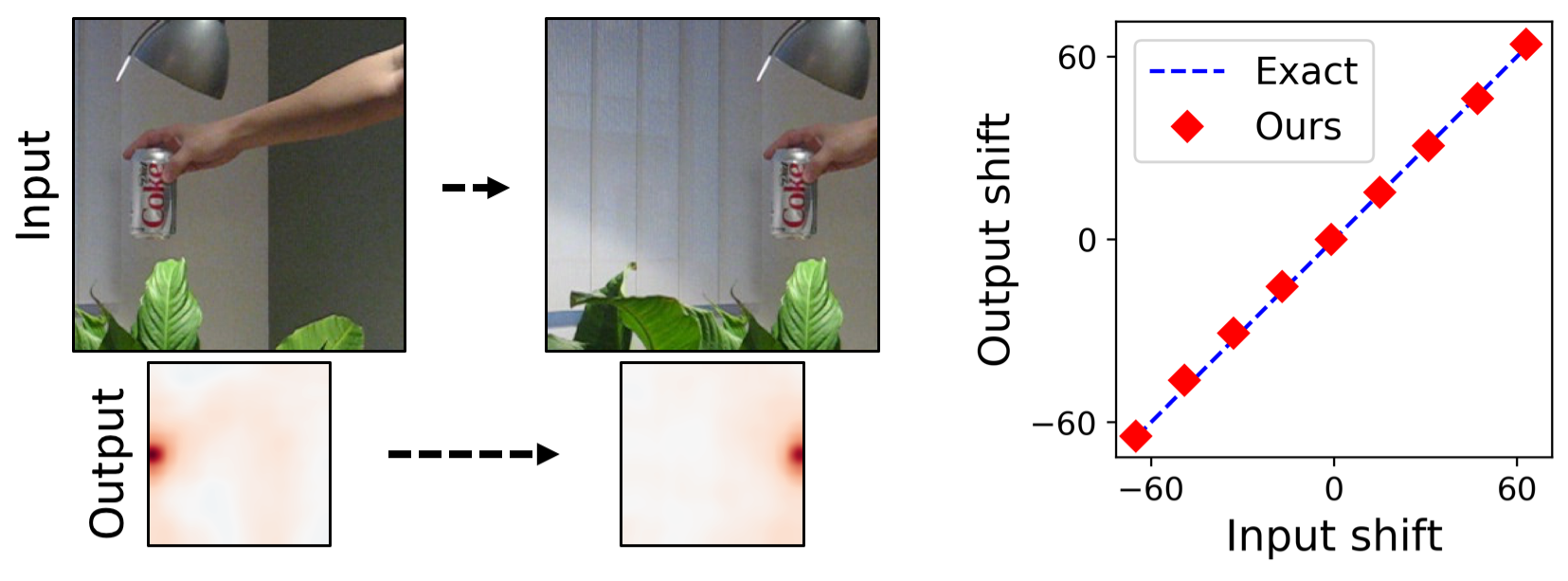}
    \end{center}
    \caption{Left: two samples from the simulated sequence. The input image is a translated and cropped version of the source image. The output is the heatmap produced by the proposed model. The red color represents the place where the object is detected. Right: correspondence between the input and the output shifts.}
    \label{fig:translation}
\end{figure}

Given two functions $f_1, f_2$ of scale and translation the non-paramteric scale convolution is defined as follows:
\begin{equation}
    \label{eq:supp_scale_conv_nonparam}
    [f_1 \star_H f_2](s, t) = L_{s^{-1}}[L_s[f_1] \star f_2](t)
\end{equation}

\begin{lemma}
A function given by \Eqref{eq:supp_scale_conv_nonparam} is equivariant under scale-translation.
\end{lemma}

\begin{figure*}[!th]
    \begin{center}
        \includegraphics[width=0.95\linewidth]{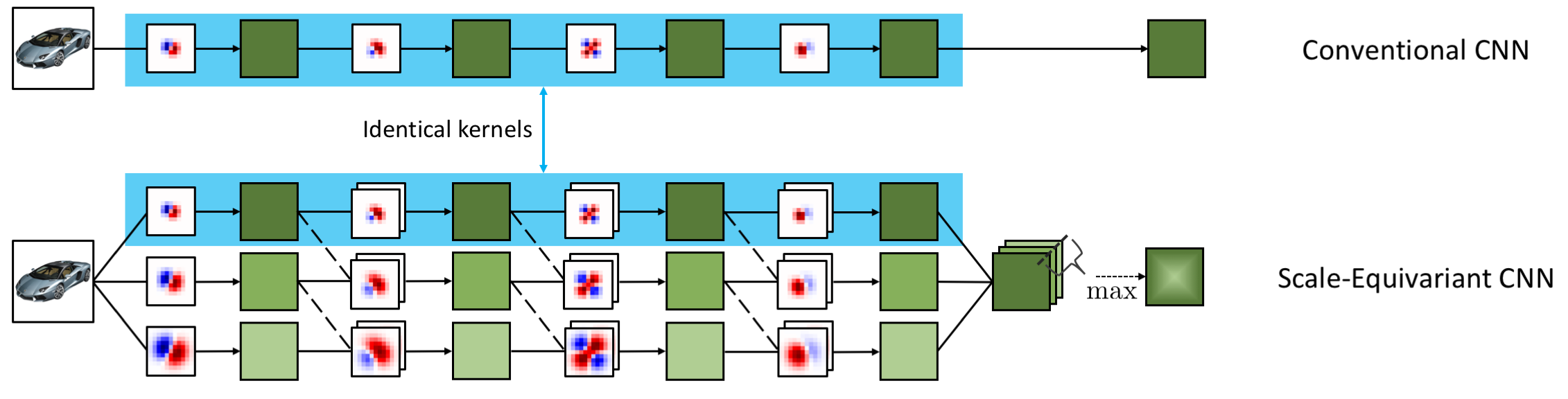}
    \end{center}
    \caption{The visualization of the weight initialization scheme from a pretrained model. Dashed connections are initialized with $0$.}
    \label{fig:supp_weight_init}
\end{figure*}

\begin{proof}
A function given by \Eqref{eq:supp_scale_conv_nonparam} is equivariant under scale transformations of $f_1$, indeed
\begin{equation}
\begin{split}
   \label{eq:supp_nonparam_proof_1}
    [L_{\hat{s}}[f_1] \star_H f_2](s, t) 
    &= L_{s^{-1}}[L_{s\hat{s}}[f_1] \star f_2](t) \\
    &= L_{\hat{s}}L_{(s\hat{s})^{-1}}[L_{s\hat{s}}[f_1] \star f_2](t) \\
    &= L_{\hat{s}}[f_1 \star_H f_2](s\hat{s}, t)
\end{split}
\end{equation}

For a pair of scale and translation $s,\hat{t}$ we have the following property of the joint transformation $L_{s}T_{\hat{t}} = T_{\hat{t}s}L_{s}$ from \cite{sosnovik2019scale}, where $T_{\hat{t}}$ is the translation operator defined as $T_{\hat{t}}[f](t) = f(t-\hat{t})$. Now we can show the following:
\begin{equation}
\begin{split}
   \label{eq:supp_nonparam_proof_2}
    [T_{\hat{t}}[f_1] \star_H f_2](s, t) 
    &= L_{s^{-1}}[L_{s}[T_{\hat{t}}[f_1]] \star f_2](t) \\
    &= L_{s^{-1}}[T_{\hat{t}s}L_{s}[f_1] \star f_2](t) \\
    &= L_{s^{-1}}T_{\hat{t}s}[L_{s}[f_1] \star f_2](t) \\
    &= T_{\hat{t}}L_{s^{-1}}[L_{s}[f_1] \star f_2](t) \\
    &= T_{\hat{t}}[f_1 \star_H f_2](t) \\
\end{split}
\end{equation}
 Therefore, a function given by \Eqref{eq:supp_scale_conv_nonparam} is also equivariant under translations of $f_1$. The equivariance of the function with respect to a joint transformation follows from the equivariance to each of the transformations separately \cite{sosnovik2019scale}. 

We proved the equivariance with respect to $f_1$. The proof with respect to $f_2$ is analogous. 
\end{proof}

\section{Weight initialization}

The proposed weight initialization scheme from a pretrained model is depicted in \Figref{fig:supp_weight_init}.
\section{Experiments}

\subsection{Padding}
We conduct an experiment to verify that the proposed padding technique does not violate translation equivariance of convolutional trackers. We choose an image and select a sequence of translated and cropped windows inside of it. We process this sequence with a deep model that consists of the proposed convolutional layers and follows the inference procedure described in \cite{zhang2019deeper}. We derive the predicted location of the object and compare its value to the input shift. \Figref{fig:translation} demonstrates that the input and the output translations have nearly identical values.

\subsection{Translating-Scaling MNIST}

\begin{table}[h]
\centering
\begin{tabular}{|c|c|c|}
\hline
Stage & SiamFC & SE-SiamFC \\ \hline
\multirow{2}{*}{Conv1} & \multicolumn{2}{c|}{\multirow{2}{*}{$\begin{bmatrix}3 \times 3, 96 , s=2 \\\end{bmatrix}$}} \\
 & \multicolumn{2}{c|}{} \\ \hline
\multirow{2}{*}{Conv2} & \multicolumn{2}{c|}{\multirow{2}{*}{$\begin{bmatrix}3 \times 3, 128 , s=2 \\\end{bmatrix}$}} \\
 & \multicolumn{2}{c|}{} \\ \hline
\multirow{2}{*}{Conv3} & \multicolumn{2}{c|}{\multirow{2}{*}{$\begin{bmatrix}3 \times 3, 256 , s=2 \\\end{bmatrix}$}} \\
 & \multicolumn{2}{c|}{} \\ \hline
\multirow{2}{*}{Conv4} & \multicolumn{2}{c|}{\multirow{2}{*}{$\begin{bmatrix}3 \times 3, 256 , s=1 \\\end{bmatrix}$}} \\
 & \multicolumn{2}{c|}{} \\ \hline
\multirow{2}{*}{Connect.} & \multirow{2}{*}{Cross-correlation} & \multirow{2}{*}{\begin{tabular}[c]{@{}c@{}}Non-parametric\\ scale-convolution\end{tabular}} \\
 &  &  \\ \hline
\multirow{2}{*}{\# Params} & 
\multirow{2}{*}{$999$ K}& 
\multirow{2}{*}{$999$ K}\\
 & & \\ \hline
\end{tabular}
\caption{Architectures used in T/S-MNIST experiment. All convolutions in SE-SiamFC are scale-convolutions.}
\label{tab:mnist_arch}
\end{table}

For both T-MNIST and S-MNIST, we use architectures described in Table \ref{tab:mnist_arch}. 2D BatchNorm and ReLU are inserted after each of the convolutional layers except the last one. We do not use max pooling to preserve strict translation-equivariance. 

We train both models for 50 epochs using SGD with a mini-batch of $8$ images and exponentially decay the learning rate from $10^{-2}$ to $10^{-5}$. We set the momentum to $0.9$ and the weight decay to $0.5^{-4}$. A binary cross-entropy loss as in \cite{bertinetto2016fully} is used. The inference algorithm is the same for both SiamFC and SE-SiamFC and follows the original implementation \cite{bertinetto2016fully}.

\subsection{OTB and VOT}
For OTB and VOT experiments we used architectures described in Table \ref{tab:otbvot_arch}. We use the baseline \cite{zhang2019deeper} with Cropping Inside Residual (CIR) units. SE-SiamFC is constructed directly from the baseline as described in the paper. In Table \ref{tab:otbvot_arch} the kernel size refers to the smallest scale $\sigma = 1$ in the network. The sizes of the kernels, which correspond to bigger scales are $9 \times 9$ for Conv1 and $5 \times 5$ for other layers. \Figref{fig:qual_supp} gives a qualitative comparison of the proposed method and the baseline. 

\begin{table}[ht]
\begin{tabular}{|c|c|c|}
\hline
Stage & SiamFC+ & SE-SiamFC \\ \hline
\multirow{2}{*}{Conv1} & \multirow{2}{*}{$\begin{bmatrix}7 \times 7, 64 , s=2 \\\end{bmatrix}$} & \multirow{2}{*}{$\begin{bmatrix}7 \times 7, 64 , s=2 \\\end{bmatrix}$} \\
 &  &  \\ \hline
\multirow{6}{*}{Conv2} & \multicolumn{2}{c|}{\multirow{2}{*}{max pool $\begin{bmatrix}2 \times 2, s=2 \\\end{bmatrix}$}} \\
 & \multicolumn{2}{c|}{} \\ \cline{2-3} 
 & \multirow{4}{*}{$\begin{bmatrix}1 \times 1, 64 \\ 3 \times 3, 64 \\ 1 \times 1, 256 \\ \end{bmatrix} \times 3$} & \multirow{4}{*}{$\begin{bmatrix}1 \times 1, 64, i=2 \\ 3 \times 3, 64 \\ 1 \times 1, 256 \\ \end{bmatrix} \times 3$} \\
 &  &  \\
 &  &  \\
 &  &  \\ 
 \hline
\multirow{4}{*}{Conv3} & \multirow{4}{*}{$\begin{bmatrix}1 \times 1, 128 \\ 3 \times 3, 128 \\ 1 \times 1, 512 \\ \end{bmatrix} \times 3$} & \multirow{4}{*}{$\begin{bmatrix}1 \times 1, 128, \textit{sp} \\ 3 \times 3, 128 \\ 1 \times 1, 512 \\ \end{bmatrix} \times 3$} \\
 &  &  \\
 &  &  \\
 &  &  \\ 
 \hline
 Connect. & Cross-correlation & \begin{tabular}[c]{@{}c@{}}Non-parametric\\ scale-convolution\end{tabular} \\ \hline
\multirow{2}{*}{\# Params} & 
\multirow{2}{*}{$1.44$ M}& 
\multirow{2}{*}{$1.45$ M}\\
 & & \\ \hline
\end{tabular}
\caption{Architectures used in OTB/VOT experiments. All convolutions in SE-SiamFC are scale-convolutions. \textit{s} refers to stride, \textit{sp} denotes scale pooling, \textit{i} --- is the size of the kernel in a scale dimension.}
\label{tab:otbvot_arch}
\end{table}

\begin{figure*}[!b]
    \begin{center}
        \includegraphics[width=\linewidth]{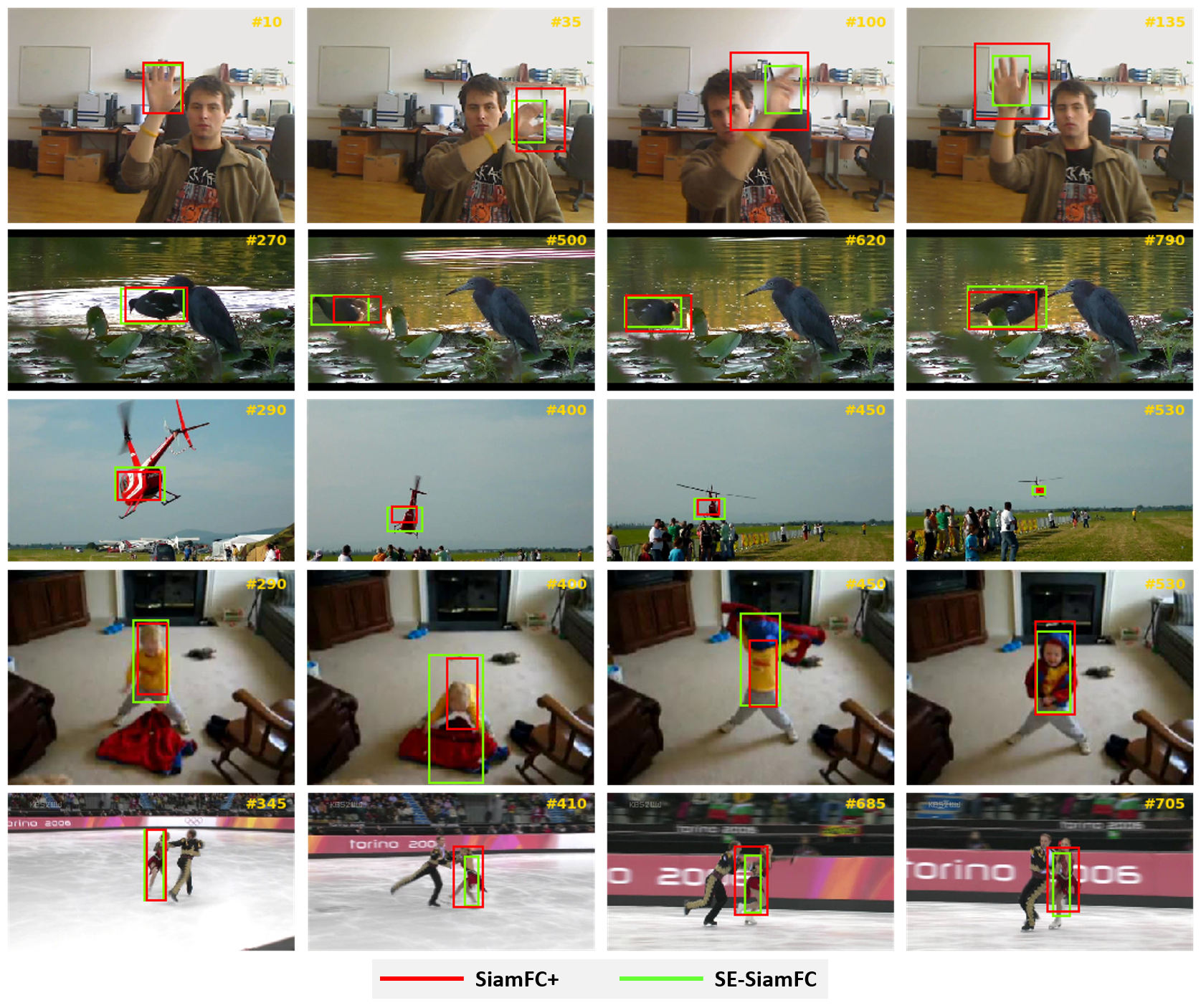}
    \end{center}
    \caption{Qualitative comparison of SE-SiamFC with SiamFC+ on VOT2016/2017 sequences.}
    \label{fig:qual_supp}
\end{figure*}

\end{document}